%% file: main.tex
\documentclass{article}

\usepackage{microtype}
\usepackage{graphicx}
\usepackage{subfigure}
\usepackage{booktabs} 
\usepackage[utf8]{inputenc}
\usepackage{pgfplots}
\usepackage{multirow}
\usepackage{amsmath,amsthm,amssymb}
\usepackage{stmaryrd}
\usepackage{relsize}
\usepackage{color, colortbl}

\usepackage{hyperref}

\usepackage[accepted]{icml2019}

\definecolor{myrefcolor}{RGB}{0,20,115}

\newcommand{\cH}{\mathcal{H}}
\newcommand{\cM}{\mathcal{M}}
\newcommand{\cW}{\mathcal{W}}
\newcommand{\cD}{\mathcal{D}}
\newcommand{\cE}{\mathcal{E}}
\newcommand{\cZ}{\mathcal{Z}}
\newcommand{\cX}{\mathcal{X}}

\newcommand{\R}{\mathbf{R}}
\newcommand{\bi}{\mathbf{i}}

\newcommand{\bN}{\mathbf{N}}

\newcommand{\dd}{\mathrm{d}}
\newcommand{\bbM}{\mathbb{M}}
\newcommand{\bbF}{\mathbb{F}}
\newcommand{\Lip}{\operatorname{Lip}}
\newcommand{\norm}[1]{\| #1 \|}
\newcommand{\normc}[1]{| #1 |}
\newcommand{\iprod}[2]{\langle #1, #2 \rangle}

\newcommand{\weaks}{\overset{\ast}{\rightharpoonup}}
\newcommand{\spt}{\operatorname{spt}}
\newcommand{\LM}{\mathbf{\Lambda}}

\newtheorem{defi}{Definition}

\newtheorem{prop}{Proposition}

\newtheorem{asm}{Assumption}

\newcommand{\measurerestr}{%
  \,\raisebox{-.127ex}{\reflectbox{\rotatebox[origin=br]{-90}{$\lnot$}}}\,%
}

\icmltitlerunning{Flat Metric Minimization with Applications in Generative Modeling}

\graphicspath{{figures/flatnorm/}{figures/}}

\begin{document}

\twocolumn[
\icmltitle{Flat Metric Minimization with Applications in Generative Modeling}



\icmlsetsymbol{equal}{*}

\begin{icmlauthorlist}
\icmlauthor{Thomas M\"ollenhoff}{tum}
\icmlauthor{Daniel Cremers}{tum}
\end{icmlauthorlist}

\icmlaffiliation{tum}{Department of Informatics, Technical University of Munich, Garching, Germany}

\icmlcorrespondingauthor{Thomas M\"ollenhoff}{thomas.moellenhoff@tum.de}

\icmlkeywords{Generative Adversarial Networks, Geometric Measure Theory, Representation Learning, Disentangled Representations}

\vskip 0.3in
]



\printAffiliationsAndNotice{}  

\begin{abstract}
  We take the novel perspective to view data not 
  as a probability distribution but rather as a current. Primarily studied in
  the field of geometric measure theory, $k$-currents are continuous
  linear functionals acting on compactly supported smooth differential forms
  and can be understood as a generalized notion of oriented $k$-dimensional
  manifold. By moving from distributions (which are $0$-currents) to $k$-currents,
  we can explicitly orient the data by attaching a $k$-dimensional
  tangent plane to each sample point.
  Based on the flat metric which is a fundamental distance between currents,
  we derive FlatGAN, a formulation in the spirit of generative
  adversarial networks but generalized to $k$-currents.
  In our theoretical contribution we prove that the flat metric between a parametrized current and a
  reference current is Lipschitz continuous in the parameters.
  In experiments, we show that the proposed
  shift to $k>0$ leads to interpretable and
  disentangled latent representations which behave equivariantly to the
  specified oriented tangent planes.
\end{abstract}

\input{sections/intro.tex}
\input{sections/gmt.tex}

\input{sections/flatnorm.tex}

\input{sections/gan.tex}
\input{sections/impl.tex}

\input{sections/experiments.tex}
\input{sections/discussion.tex}

{
\bibliography{sections/reference.bib}
\bibliographystyle{icml2019}
}

\clearpage

\renewcommand\thesection{\Alph{section}}
\input{sections/appendix.tex}

\end{document}

%% file: sections/intro.tex
\begin{figure}[t!]
  \begin{center}
      {\includegraphics[width=0.99\linewidth]{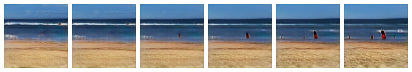}}\\[-0.14cm]
      {\includegraphics[width=\linewidth]{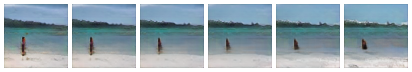}}\\[-0.14cm]
      {\includegraphics[width=\linewidth]{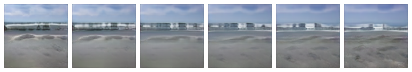}}\\[-0.14cm]
      {\includegraphics[width=\linewidth]{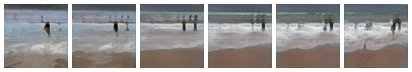}}\\
      from left to right we vary the latent code $z_1$ (time)
  \end{center}
  \caption{Discovering the arrow of time by training a generative model with
    the proposed formalism on the tinyvideos dataset \cite{videoGAN}. The approach we introduce allows one to learn latent representations 
    which behave equivariantly to specified tangent vectors (here: difference of two successive video frames). }
  \label{fig:teaser}
\end{figure}

\begin{figure*}[t!]
  \centering
  \begin{center}
    {\def\svgwidth{\linewidth} 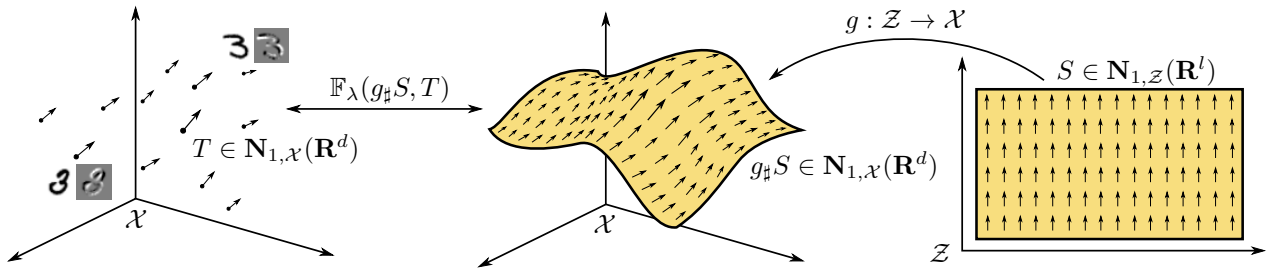}
  \end{center}
  \caption{\textbf{Illustration of the proposed idea.} We suggest the novel perspective to view observed data (here the MNIST dataset) as a $k$-current $T$, shown as the dots with attached arrows on the left. The arrows indicate the oriented tangent space, and we selected $k = 1$ to be rotational deformation. We propose to minimize the flat distance of $T$ to the \emph{pushforward} $g_\sharp S$ (shown in the middle) of a current $S$ on a low-dimensional latent space $\cZ$ (right) with respect to a ``generator'' map $g : \cZ \to \cX$. For $0$-currents (no selected tangent vectors) and sufficiently large $\lambda$, the proposed ``FlatGAN'' formulation specializes to the Wasserstein GAN \cite{ACB17}.}
  \label{fig:nutshell}
\end{figure*}

\section{Introduction}
\label{sec:intro}
This work is concerned with the problem of representation learning, which has important consequences for many tasks in artificial intelligence, cf. the work of \citet{bengio2013representation}. More specifically,
our aim is to learn representations which behave
equivariantly with respect to selected transformations of the data.
Such variations are often known beforehand and could for example describe changes in stroke width or rotation of a digit,
changes in viewpoint or lighting in a three-dimensional scene but also
the \emph{arrow of time} \cite{pickup2014seeing, wei2018learning} in time-series, describing how a video
changes from one frame to the next, see Fig.~\ref{fig:teaser}.

We tackle this problem by introducing a novel formalism
based on \emph{geometric measure theory} \cite{federer}, which we find to
be interesting in itself.
To motivate our application in generative modeling, recall
the manifold hypothesis which states that
the distribution of real-world data tends to
concentrate nearby a low-dimensional manifold, see \citet{fefferman2016testing} and the references therein.
Under that hypothesis, a possible unifying view on
prominent methods in unsupervised and representation learning such as generative adversarial networks (GANs) \cite{Goo+14} and variational auto-encoders (VAEs) \cite{KiWe14,rezende2014stochastic} is the following: both approaches aim to approximate the true distribution
concentrating near the manifold with a distribution on some low-dimensional latent space $\cZ \subset \R^l$ that is pushed through a decoder or generator $g : \cZ \to \cX$
mapping to the (high-dimensional) data space $\cX \subset \R^d$ \cite{GPC17_view,BALO17}.

We argue that treating data as a distribution potentially ignores useful available geometric information such as orientation and tangent vectors to the data manifold.
Such tangent vectors describe the aforementioned local variations or pertubations.
Therefore we postulate that
\emph{data should not be viewed as a distribution but rather as a $k$-current}. 

We postpone the definition of $k$-currents \cite{deRham} to Sec.~\ref{sec:gmt},
and informally think of them as distributions over $k$-dimensional oriented planes.
For the limiting case $k=0$, currents simply reduce to distributions in the sense of \citet{schwartz} and positive $0$-currents with unit mass are probability measures.
A seminal work in the theory of currents was written by \citet{federer1960normal}, which established
compactness theorems for subsets of currents (\emph{normal} and \emph{integral currents}).
In this paper, we will work in the space of normal $k$-currents with compact support in $\cX \subset \R^d$, denoted by $\bN_{k,\cX}(\R^d)$.

Similarly as probabilistic models build upon $f$-divergences \cite{csiszar2004information}, integral probability metrics \cite{sriperumbudur2012empirical} or more general optimal transportation
related divergences \cite{PeCu18,feydy2018interpolating}, we require a sensible notion to measure
``distance'' between $k$-currents.

In this work, we will focus on the flat norm\footnote{The terminology ``flat'' carries no geometrical significance and refers to Whitney's use of musical notation flat $\normc{\cdot}^\flat$ and sharp
  $\normc{\cdot}^\sharp$.} 
due to \citet{whitney1957geometric}. To be precise, we consider a scaled variant introduced and studied by \citet{morgan2007,vixie2010multiscale}. 
This choice is motivated in Sec.~\ref{sec:flat}, where we show that the flat norm enjoys certain attractive properties similar to the celebrated Wasserstein distances.
For example, it metrizes the weak$^*$-convergence for normal currents.

A potential alternative to the flat norm are kernel metrics on spaces of currents \cite{vaillant2005surface, glaunes2008large}.
These have been proposed for diffeomorphic registration, but kernel distances on distributions have also been sucessfully employed for generative modeling, see \citet{li2017mmd}. 
Constructions similar to the Kantorovich relaxation in optimal transport but generalized to $k$-currents 
recently appeared in the context of convexifications for certain variational problems \cite{moellenh19}.

\section{Related Work}
Our main idea is illustrated in Fig.~\ref{fig:nutshell}, which
was inspired from the optimal transportation point of view on GANs given by \citet{GPC17_view}.

Tangent vectors of the data manifold, either prespecified \cite{simard1992tangent,simard1998transformation,fraser2003incorporating} or learned with
a contractive autoencoder \cite{rifai2011manifold}, have been used to train classifiers that aim to be \emph{invariant} to changes relative to the data manifold. In contrast to these works, we
use tangent vectors to learn interpretable representations and a generative model that aims to be \emph{equivariant}.
The principled introduction of tangent $k$-vectors into
probabilistic generative models is one of our main contributions.

Various approaches to learning informative or disentangled latent representations in a
completely unsupervised fashion exist \cite{schmidhuber,bVAE,CDHSSA16,factorVAE}. Our approach is
orthogonal to these works, as specifying tangent vectors further encourages informative
representations to be learned. For example, our GAN formulation could be
combined with a mutual information term as in InfoGAN \cite{CDHSSA16}.

Our work is more closely related to semi-supervised approaches on learning disentangled latent representations, which similarly also require some form of knowledge of the underlying factors \cite{hinton,denton,mathieu,siddharth} and also to conditional GANs \cite{mirza2014conditional,odena2017}.
However, the difference is the connection to geometric measure theory which we believe to be
completely novel, and our specific FlatGAN formulation that seamlessly extends the Wasserstein GAN \cite{ACB17}, cf. Fig.~\ref{fig:nutshell}.

Since the concepts we need from geometric measure theory are not commonly used in machine learning, we briefly review
them in the following section.

%% file: figures/manifold_rev.pdf_tex
\begingroup%
  \makeatletter%
  \providecommand\color[2][]{%
    \errmessage{(Inkscape) Color is used for the text in Inkscape, but the package 'color.sty' is not loaded}%
    \renewcommand\color[2][]{}%
  }%
  \providecommand\transparent[1]{%
    \errmessage{(Inkscape) Transparency is used (non-zero) for the text in Inkscape, but the package 'transparent.sty' is not loaded}%
    \renewcommand\transparent[1]{}%
  }%
  \providecommand\rotatebox[2]{#2}%
  \ifx\svgwidth\undefined%
    \setlength{\unitlength}{1198.61445389bp}%
    \ifx\svgscale\undefined%
      \relax%
    \else%
      \setlength{\unitlength}{\unitlength * \real{\svgscale}}%
    \fi%
  \else%
    \setlength{\unitlength}{\svgwidth}%
  \fi%
  \global\let\svgwidth\undefined%
  \global\let\svgscale\undefined%
  \makeatother%
  \begin{picture}(1,0.20301614)%
    \put(0,0){\includegraphics[width=\unitlength,page=1]{manifold_rev.pdf}}%
    \put(0.26907114,0.13271003){\color[rgb]{0,0,0}\makebox(0,0)[lb]{\smash{$\bbF_\lambda(g_{\sharp}S, T)$}}}%
    \put(0,0){\includegraphics[width=\unitlength,page=2]{manifold_rev.pdf}}%
    \put(0.66861174,0.18612746){\color[rgb]{0,0,0}\makebox(0,0)[lb]{\smash{$g : \cZ \to \cX$}}}%
    \put(0,0){\includegraphics[width=\unitlength,page=3]{manifold_rev.pdf}}%
    \put(0.7337435,0.00444269){\color[rgb]{0,0,0}\makebox(0,0)[lb]{\smash{$\cZ$}}}%
    \put(0.83238816,0.14640634){\color[rgb]{0,0,0}\makebox(0,0)[lb]{\smash{$S \in \bN_{1,\cZ}(\R^l)$}}}%
    \put(0,0){\includegraphics[width=\unitlength,page=4]{manifold_rev.pdf}}%
    \put(0.4731679,0.02637017){\color[rgb]{0,0,0}\makebox(0,0)[lb]{\smash{$\cX$}}}%
    \put(0.59488921,0.07278736){\color[rgb]{0,0,0}\makebox(0,0)[lb]{\smash{$g_\sharp S \in \bN_{1,\cX}(\R^d)$}}}%
    \put(0,0){\includegraphics[width=\unitlength,page=5]{manifold_rev.pdf}}%
    \put(0.1098471,0.03039349){\color[rgb]{0,0,0}\makebox(0,0)[lb]{\smash{$\cX$}}}%
    \put(0.16245214,0.08584937){\color[rgb]{0,0,0}\makebox(0,0)[lb]{\smash{$T \in \bN_{1,\cX}(\R^d)$}}}%
    \put(0,0){\includegraphics[width=\unitlength,page=6]{manifold_rev.pdf}}%
  \end{picture}%
\endgroup%

%% file: sections/gmt.tex
\section{Geometric Measure Theory}
\label{sec:gmt}
The book by \citet{federer} is still the formidable, definitive reference on the subject. As
a more accessible introduction we recommend \cite{KP08} or \cite{morgan2016geometric}. While our aim
is to keep the manuscript self-contained, we invite the interested reader to
consult Chapter~4 in \cite{morgan2016geometric}, which in turn refers to
the corresponding chapters in the book of \citet{federer} for more details.

\subsection{Grassmann Algebra}

\paragraph{Notation.}
Denote $\{ e_1, \hdots, e_d \}$ a basis of $\R^d$ with dual basis $\{ \dd x_1, \hdots, \dd x_d \}$ such
that $\dd x_i : \R^d \to \R$ is the linear functional that maps every $x = (x_1, \hdots, x_d)$ to the $i$-th component $x_i$.
For $k \leq d$, denote $I(d, k)$ as the ordered multi-indices $\bi = (i_1, \hdots, i_k)$ with $1 \leq i_1 < \hdots < i_k \leq d$.

One can multiply vectors in $\R^d$ to obtain a new object:
\begin{equation}
  \xi = v_1 \wedge \hdots \wedge v_k,
  \label{eq:sv}
\end{equation}
called a $k$-vector $\xi$ in $\R^d$. The wedge (or exterior) product $\wedge$ is characterized by
multilinearity
\begin{equation}
  \begin{aligned}
    &c v_1 \wedge v_2 = v_1 \wedge c v_2 = c (v_1 \wedge v_2), ~ \text{ for } c \in \R,\\
    &(u_1 + v_1) \wedge (u_2 + v_2) = \\
    &\qquad u_1 \wedge u_2 + u_1 \wedge v_2 + v_1 \wedge u_2 + v_1 \wedge v_2,
  \end{aligned}
  \label{eq:gm1}
\end{equation}
and it is alternating
\begin{equation}
  u \wedge v = -v \wedge u, \quad u \wedge u = 0.
  \label{eq:gm2}
\end{equation}
In general, any $k$-vector can be written as
\begin{equation}
  \xi = \sum_{\bi \in I(d,k)} a_\bi \cdot e_{i_1} \wedge \hdots \wedge e_{i_k} = \sum_{\bi \in I(d,k)} a_\bi \cdot e_\bi,
  \label{eq:vecsum}
\end{equation}
for coefficients $a_\bi \in \R$. The vector space of $k$-vectors is denoted by $\LM_k \R^d$ and has
dimension $\binom{d}{k}$. We define for two $k$-vectors $v = \sum_\bi a_\bi e_\bi$, $w = \sum_\bi b_\bi e_\bi$ an inner product $\iprod{v}{w} = \sum_\bi a_{\bi} b_{\bi}$ and the Euclidean norm $\normc{v} = \sqrt{\iprod{v}{v}}$.

A simple (or decomposable) $k$-vector is any $\xi \in \LM_k \R^d$ that can be written using products of $1$-vectors.
Simple $k$-vectors such as \eqref{eq:sv} are uniquely determined by the $k$-dimensional space spanned by the $\{ v_i \}$, their orientation and the norm $|v|$ corresponding to the area of the parallelotope spanned by the $\{ v_i \}$. Simple $k$-vectors with unit norm can therefore be thought of as oriented $k$-dimensional subspaces and the rules \eqref{eq:gm1}-\eqref{eq:gm2} can be thought of as equivalence relations. 

It turns out that the inner product of two simple $k$-vectors can be computed by the $k \times k$-determinant
\begin{equation}
  \iprod{w_1 \wedge \hdots \wedge w_k}{v_1 \wedge \hdots \wedge v_k} = \det \left( W^\top V \right),
  \label{eq:det}
\end{equation}
where the columns of $W \in \R^{d \times k}$, $V \in \R^{d \times k}$ contain the individual $1$-vectors.
This will be useful later for our practical implementation.

Not all $k$-vectors are simple. An illustrative example is $e_1 \wedge e_2 + e_3 \wedge e_4 \in \LM_2 \R^4$, which describes two $2$-dimensional subspaces in $\R^4$ intersecting only at zero.

The dual space of $\LM_k \R^d$ is denoted as $\LM^k \R^d$, and its elements are called $k$-covectors. They are similarly represented as \eqref{eq:vecsum} but with dual basis $\dd x_{\bi}$. Analogously to the previous page, we can define an inner product between $k$-vectors and $k$-covectors.
Next to the Euclidean norm $|\cdot|$, we define two additional norms due to \citet{whitney1957geometric}.
\begin{defi}[Mass and comass]
The {comass norm} defined for $k$-covectors $w \in \LM^k \R^n$
is given by
\begin{equation}
  \norm{w}^* = \sup \left \{ \iprod{w}{v} : v \text{ is simple }, |v| = 1 \right \},
  \label{eq:comass}
\end{equation}
and the {mass norm} for $v \in \LM_k \R^n$ is given by
\begin{equation}
  \begin{aligned}
    \norm{v} &= \sup \left \{ \iprod{v}{w} : \norm{w}^* \leq 1 \right \} \\
    &= \inf \left \{ \sum_i | \xi_i | : \xi_i \text{ are simple}, v = \sum_i \xi_i \right \}.
  \end{aligned}
  \label{eq:mass}
\end{equation}
\end{defi}
The mass norm is by construction the largest norm that agrees with the Euclidean norm on simple $k$-vectors. For the non-simple $2$-vector from before, we compute
\begin{equation}
  \norm{e_1 \wedge e_2 + e_3 \wedge e_4} = 2, ~ \normc{e_1 \wedge e_2 + e_3 \wedge e_4} = \sqrt{2}.
\end{equation}
Interpreting the non-simple vector as two tangent planes, we see that the mass norm gives the correct
area, while the Euclidean norm underestimates it. The comass $\norm{\cdot}^*$ will be used later to define 
the mass of currents and the flat norm.

\subsection{Differential Forms}
In order to define currents, we first need to introduce differential forms.
A differential $k$-form is a $k$-covectorfield $\omega : \R^d \to \LM^k \R^d$. The
support $\spt \omega$ is defined as the closure of the set $\{ x \in \R^d : \omega(x) \neq 0 \}$.

Differential forms allow one to perform coordinate-free integration over
oriented manifolds. Given some manifold $\cM \subset \R^d$, possibly with boundary, an
\emph{orientation} is a continuous map $\tau_\cM : \cM \to \LM_k \R^d$ which assigns to each point a simple $k$-vector
with unit norm that spans the tangent space at that point. Integration of a differential form over an oriented
manifold $\cM$ is then defined by:
\begin{equation}
  \int_\cM \omega = \int_\cM \iprod{\omega(x)}{\tau_\cM(x)} \, \dd \mathcal{H}^k(x),
  \label{eq:integral}
\end{equation}
where the second integral is the standard Lebesgue integral with respect to the $k$-dimensional
Hausdorff measure $\cH^k$ restricted to $\cM$, i.e., $(\cH^k \measurerestr \cM)(A) = \cH^k(A \cap \cM)$.
The $k$-dimensional Hausdorff measure assigns to sets in $\R^d$ their $k$-dimensional volume, see Chapter 2 in \citet{morgan2016geometric} for a nice illustration. For $k = d$ the Hausdorff measure coincides with the Lebesgue measure.

The exterior derivative of a differential $k$-form is the
$(k+1)$-form $d \omega : \R^d \to \LM^{k+1} \R^d$ defined by
\begin{equation}
  \begin{aligned}
    &\iprod{d \omega(x)}{v_1 \wedge \hdots \wedge v_{k+1}} =
    \lim_{h \to 0} \frac{1}{h^{k+1}} \int_{\partial P} \omega, 
  \end{aligned}
  \label{eq:exterior_diff}
\end{equation}
where $\partial P$ is the oriented boundary of the parallelotope spanned by the $\{ h v_i \}$ at point $x$.
The above definition is for example used in the textbook of \citet{hubbard}.
To get an intuition, note that for $k = 0$ this reduces to the familiar directional derivative $\iprod{d\omega(x)}{v_1} = \lim_{h \to 0} \frac{1}{h} \left( \omega(x + h v_1) - \omega(x) \right)$.
In case $\omega : \R^d \to \LM^k \R^d$ is sufficiently smooth, the limit in \eqref{eq:exterior_diff}
is given by
\begin{align}
    \label{eq:ext_diff_lim}
    &\iprod{d \omega(x)}{v_1 \wedge \hdots \wedge v_{k+1}} = \\
    &\qquad \sum_{i=1}^{k+1} (-1)^{i-1} \nabla_x \iprod{\omega(x)}{v_1 \wedge \hdots \wedge \hat v_i \wedge \hdots \wedge v_k} \cdot v_i, \notag
\end{align}
where $\hat v_i$ means that the vector $v_i$ is omitted. The formulation \eqref{eq:ext_diff_lim} will be used in the
practical implementation.
Interestingly, with \eqref{eq:integral} and \eqref{eq:exterior_diff} in mind, Stokes' theorem
\begin{equation}
  \int_{\cM} d \omega = \int_{\partial \cM} \omega,
  \label{eq:stokes}
\end{equation}
becomes almost obvious, as (informally speaking) integrating \eqref{eq:exterior_diff} one obtains \eqref{eq:stokes} since the oppositely oriented boundaries of neighbouring parallelotopes cancel each other out in the interior of $\cM$.

To define the pushforward of currents which is central to our formulation, we require the pullback of differential forms.
The pullback $g^\sharp \omega : \R^l \to \LM^k \R^l$ by a map $g : \R^l \to \R^d$ of the $k$-form $\omega : \R^d \to \LM^k \R^d$
is given by
\begin{equation}
    \iprod{g^\sharp \omega}{v_1 \wedge .. \wedge v_k} = \iprod{\omega \circ g}{D_{v_1} g \wedge .. \wedge D_{v_k} g},
    \label{eq:pullback}
\end{equation}
where $D_{v_i} g := \nabla g \cdot v_i$ and $\nabla g \in \R^{d \times l}$ is the Jacobian.
We will also require \eqref{eq:pullback} for the practical implementation.

\subsection{Currents}
We have now the necessary tools to define currents and the required operations on them,
which will be defined through duality with differential forms.
Consider the space of compactly supported and smooth $k$-forms in $\R^d$
which we denote by $\cD^k(\R^d)$. When furnished with an appropriate topology (cf. §4.1 in \citet{federer} for the details) this is a locally convex topological vector space. $k$-currents are continuous linear functionals on smooth, compactly supported differential forms, i.e., elements from the topological dual space $\cD_k(\R^d) = \cD^k(\R^d)'$. Some examples
for currents are given in Fig.~\ref{fig:currents}. The $0$-current in \textbf{(a)} could be an empirical data distribution,
and the $2$-current in \textbf{(b)} represents the data distribution with a two dimensional oriented tangent space at each data point. The $2$-current in \textbf{(c)} simply represents the set $[0,1]^2$ as an oriented manifold, its action on a differential form is given as in \eqref{eq:integral}.

A natural notion of convergence for currents is given by the weak$^*$ topology:
\begin{equation}
  T_i \weaks T ~\text{ iff } \,\, T_i(\omega) \to T(\omega), \text{ for all } \omega \in \cD^k(\R^d).
  \label{eq:weaks}
\end{equation}

\begin{figure}[t!]
  \begin{center}
    \setlength{\tabcolsep}{3pt}
    \begin{tabular}{ccc}
      \centering
      {\def\svgwidth{0.28\linewidth} 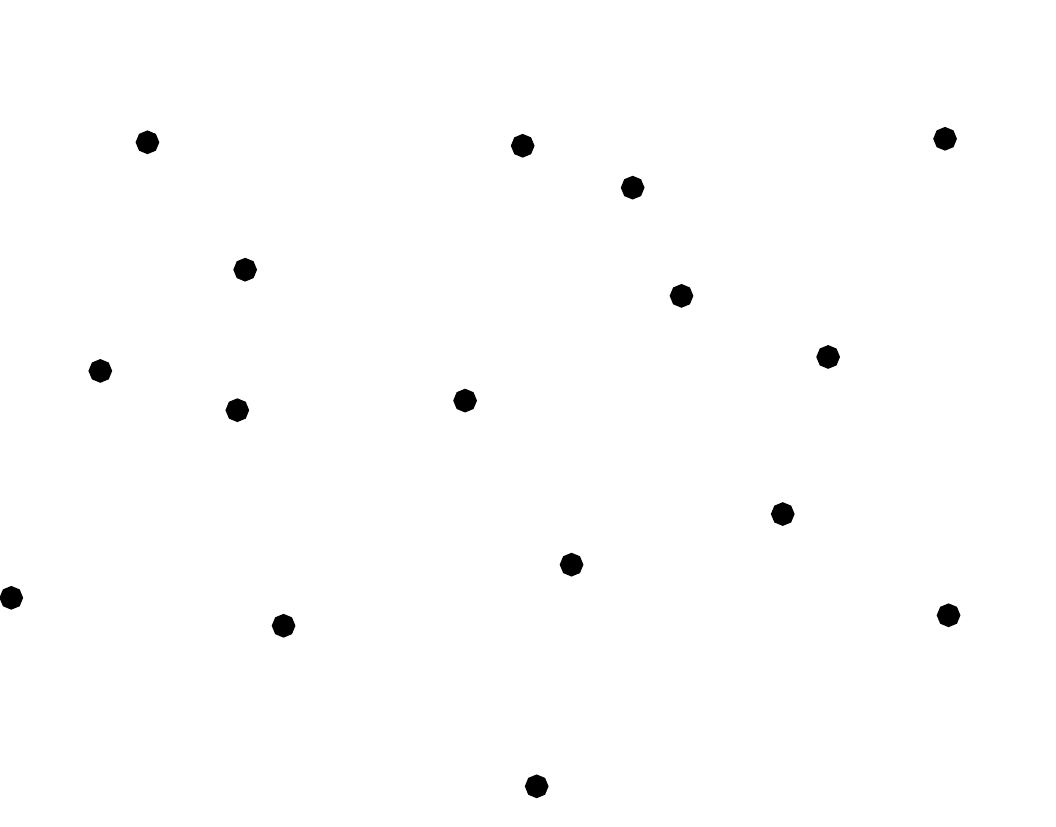}
      &{\def\svgwidth{0.28\linewidth} 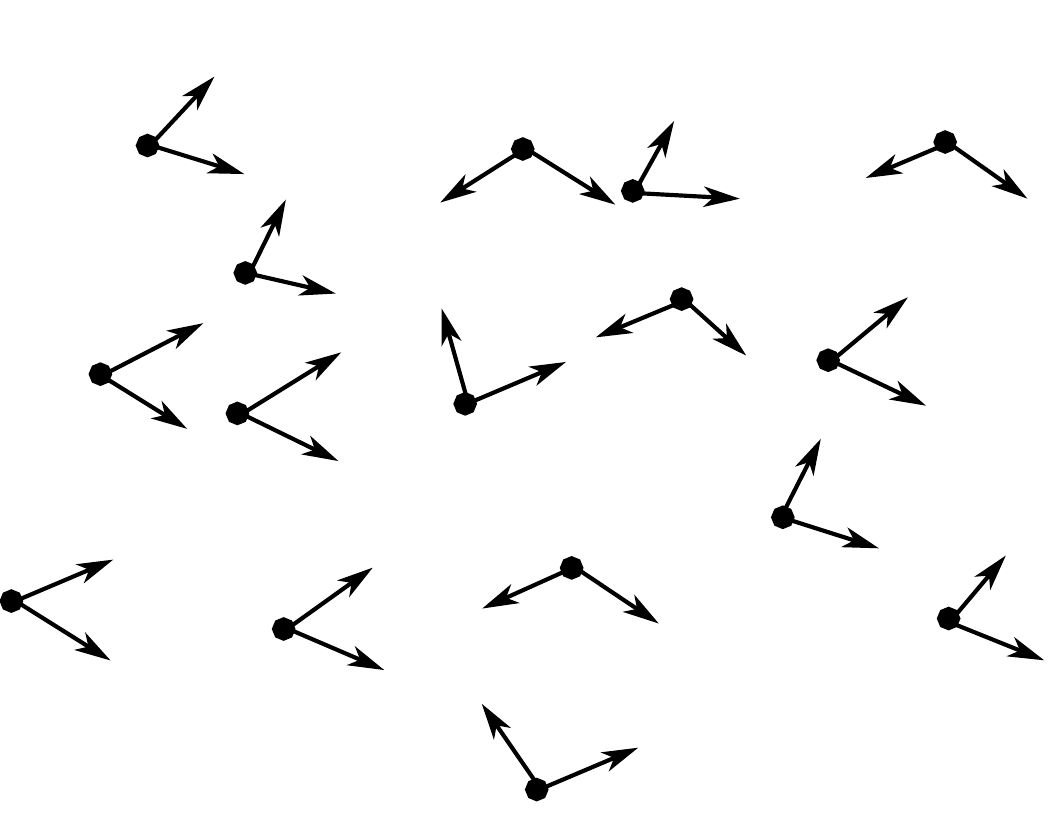}
      &{\def\svgwidth{0.22\linewidth} 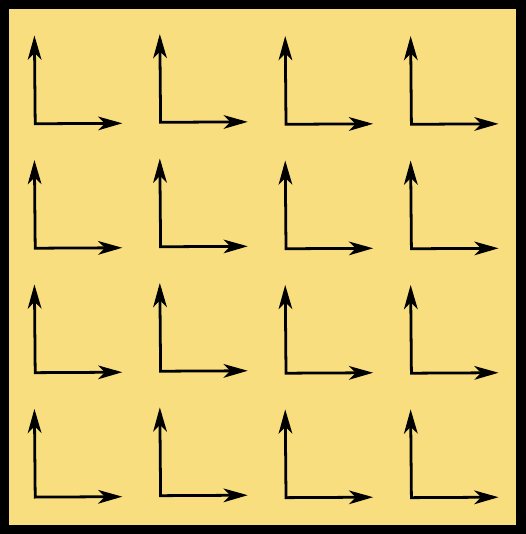}
      \\
      \textbf{(a)} $\sum_i \delta_{x_i}$
      &\textbf{(b)} $\sum_i \delta_{x_i} \wedge T_i$
      &\textbf{(c)} $\cH^2 \measurerestr [0,1]^2 \wedge e_{12}$ 
    \end{tabular}
  \end{center}
  \caption{Example of a $0$-current \textbf{(a)}, and $2$-currents \textbf{(b)}, \textbf{(c)}.}
  \label{fig:currents}
\end{figure}

The support of a current $T \in \cD_k(\R^d)$, $\spt T$, is the complement of the largest open set, so
that when testing $T$ with compactly supported forms on that open set the answer is zero. Currents with
compact support are denoted by $\cE_k(\R^d)$.
The boundary operator $\partial : \cD_k(\R^d) \to \cD_{k-1}(\R^d)$ is defined using exterior derivative
\begin{equation}
  \partial T(\omega) = T(d \omega),
  \label{eq:bdry}
\end{equation}
and Stokes' theorem \eqref{eq:stokes} ensures that this coincides with the
intuitive notion of boundary for currents which are represented by integration over manifolds
in the sense of \eqref{eq:integral}.

The pushforward of a current is defined using the pullback
\begin{equation}
  g_\sharp T(\omega) = T(g^\sharp \omega),
  \label{eq:pushfwd}
\end{equation}
where the intuition is that the pushforward transforms the current with the map $g$, see the illustration in Fig.~\ref{fig:nutshell}.

The mass of a current $T \in \cD_k(\R^d)$ is given by
\begin{equation}
  \bbM(T) = \sup \left \{ T(\omega) : \norm{\omega(x)}^* \leq 1 \right \}.
  \label{eq:mass}
\end{equation}
If the current $T$ is an oriented manifold then
the mass $\bbM(T)$ is the \emph{volume} of that manifold.
One convenient way to construct $k$-currents, is by combining a smooth $k$-vectorfield $\xi : \R^d \to \LM_k \R^d$ with a Radon measure $\mu$:
\begin{equation}
  (\mu \wedge \xi)(\psi) = \int \iprod{\xi}{\psi} \, \dd \mu, \text{ for all } \psi \in \cD^{k}(\R^d).
  \label{eq:imul}
\end{equation}
A concrete example is illustrated in Fig.~\ref{fig:currents} \textbf{(b)}, where given samples $\{ x_1, \hdots, x_N \} \subset \R^d$ and tangent $2$-vectors $\{ T_1, \hdots, T_N \} \subset \LM_2 \R^d$ a $2$-current is constructed.

For currents with finite mass there is a measure $\norm{T}$ and a map $\vec{T} : \R^d \to \LM_k \R^d$ with $\norm{\vec{T}(\cdot)} = 1$ almost everywhere so that we can represent it by integration as follows:
\begin{equation}
  T(\omega) = \int \iprod{\omega(x)}{\vec{T}(x)} \, \mathrm{d} \norm{T}(x) = \norm{T} \wedge \vec{T} \, (\omega).
  \label{eq:polar}
\end{equation}
Another perspective is that finite mass currents are simply $k$-vector valued Radon measures.
Currents with finite mass and finite boundary mass are called \emph{normal currents} \cite{federer1960normal}.
The space of normal currents with support in a compact set $\cX$ is denoted by $\bN_{k,\cX}(\R^d)$.

%% file: figures/current0.pdf_tex
\begingroup%
  \makeatletter%
  \providecommand\color[2][]{%
    \errmessage{(Inkscape) Color is used for the text in Inkscape, but the package 'color.sty' is not loaded}%
    \renewcommand\color[2][]{}%
  }%
  \providecommand\transparent[1]{%
    \errmessage{(Inkscape) Transparency is used (non-zero) for the text in Inkscape, but the package 'transparent.sty' is not loaded}%
    \renewcommand\transparent[1]{}%
  }%
  \providecommand\rotatebox[2]{#2}%
  \ifx\svgwidth\undefined%
    \setlength{\unitlength}{304.12684403bp}%
    \ifx\svgscale\undefined%
      \relax%
    \else%
      \setlength{\unitlength}{\unitlength * \real{\svgscale}}%
    \fi%
  \else%
    \setlength{\unitlength}{\svgwidth}%
  \fi%
  \global\let\svgwidth\undefined%
  \global\let\svgscale\undefined%
  \makeatother%
  \begin{picture}(1,0.78259198)%
    \put(0,0){\includegraphics[width=\unitlength,page=1]{current0.pdf}}%
  \end{picture}%
\endgroup%

%% file: figures/current1.pdf_tex
\begingroup%
  \makeatletter%
  \providecommand\color[2][]{%
    \errmessage{(Inkscape) Color is used for the text in Inkscape, but the package 'color.sty' is not loaded}%
    \renewcommand\color[2][]{}%
  }%
  \providecommand\transparent[1]{%
    \errmessage{(Inkscape) Transparency is used (non-zero) for the text in Inkscape, but the package 'transparent.sty' is not loaded}%
    \renewcommand\transparent[1]{}%
  }%
  \providecommand\rotatebox[2]{#2}%
  \ifx\svgwidth\undefined%
    \setlength{\unitlength}{304.12683105bp}%
    \ifx\svgscale\undefined%
      \relax%
    \else%
      \setlength{\unitlength}{\unitlength * \real{\svgscale}}%
    \fi%
  \else%
    \setlength{\unitlength}{\svgwidth}%
  \fi%
  \global\let\svgwidth\undefined%
  \global\let\svgscale\undefined%
  \makeatother%
  \begin{picture}(1,0.78259201)%
    \put(0,0){\includegraphics[width=\unitlength,page=1]{current1.pdf}}%
  \end{picture}%
\endgroup%

%% file: figures/current2.pdf_tex
\begingroup%
  \makeatletter%
  \providecommand\color[2][]{%
    \errmessage{(Inkscape) Color is used for the text in Inkscape, but the package 'color.sty' is not loaded}%
    \renewcommand\color[2][]{}%
  }%
  \providecommand\transparent[1]{%
    \errmessage{(Inkscape) Transparency is used (non-zero) for the text in Inkscape, but the package 'transparent.sty' is not loaded}%
    \renewcommand\transparent[1]{}%
  }%
  \providecommand\rotatebox[2]{#2}%
  \ifx\svgwidth\undefined%
    \setlength{\unitlength}{151.23557624bp}%
    \ifx\svgscale\undefined%
      \relax%
    \else%
      \setlength{\unitlength}{\unitlength * \real{\svgscale}}%
    \fi%
  \else%
    \setlength{\unitlength}{\svgwidth}%
  \fi%
  \global\let\svgwidth\undefined%
  \global\let\svgscale\undefined%
  \makeatother%
  \begin{picture}(1,1.01575069)%
    \put(0,0){\includegraphics[width=\unitlength,page=1]{current2.pdf}}%
  \end{picture}%
\endgroup%

%% file: sections/flatnorm.tex
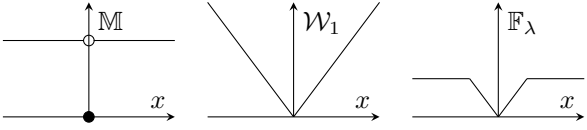
\begin{figure}[t!]
  \centering
  \begin{center}
  \begin{tabular}{ccc}
    \input{figures/massnorm.tex}
    &\input{figures/wasserstein.tex}
    &\input{figures/flatnorm.tex}
  \end{tabular}
  \end{center}
  \caption{Illustration of distances between $0$-currents on the example of two Dirac measures $\delta_x$, $\delta_0$. The flat metric $\bbF_\lambda$ has the following advantages: unlike the mass $\bbM$ it is continuous, and unlike Wasserstein-$1$ it easily generalizes to $k$-currents (see Fig.~\ref{fig:distances_2d}).}
  \label{fig:distances}
\end{figure}
\section{The Flat Metric}
\label{sec:flat}
As indicated in Fig.~\ref{fig:nutshell}, we wish to fit a current $g_\sharp S$ that is
the pushforward of a low-dimensional latent current $S$ to the current $T$ given by the data.
A more meaningful norm on currents than the mass $\bbM$ turns out to be the flat norm.
\begin{defi}[Flat norm and flat metric]
  The flat norm with scale\footnote{We picked a different convention for $\lambda$ as in \cite{morgan2007}, where it bounds the other constraint, to emphasize the connection to the Wasserstein-1 distance.} $\lambda > 0$ is defined for any $k$-current $T \in \cD_k(\R^d)$ as
  \begin{equation}
    \begin{aligned}
    \bbF_\lambda(T)
    &= \sup \bigl \{ T(\omega) ~|~ \omega \in \cD^k(\R^d), \text{ with } \\
    & \hspace{0.25cm} \norm{\omega(x)}^* \leq \lambda, \norm{d\omega(x)}^* \leq 1, \text{ for all } x \bigr \}.
  \end{aligned}
  \label{eq:flatnorm_sup}
  \end{equation}
  For $\lambda = 1$ we simply write $\bbF(\cdot) \equiv \bbF_1(\cdot)$ and $\bbF_\lambda(S, T) = \bbF_\lambda(S - T)$ will be denoted as the flat metric.
\end{defi}
The flat norm also has a primal formulation
\begin{align}
  \bbF_\lambda(T) &= \min_{B \in \cE_{k+1}(\R^d)} \lambda \bbM(T - \partial B) + \bbM(B)   \label{eq:flatnorm_min1} \\
  &= \min_{T = A + \partial B} \lambda \bbM(A) + \bbM(B),   \label{eq:flatnorm_min2}
\end{align}
where the minimum in \eqref{eq:flatnorm_min1}--\eqref{eq:flatnorm_min2} can be shown to exist, see §4.1.12 in \citet{federer}.
The flat norm is finite if $T$ is a normal current
and it can be verified that it is indeed a norm. 

To get an intuition, we compare the flat norm to the mass \eqref{eq:mass} and the Wasserstein-1 distance in Fig.~\ref{fig:distances} on the example of Dirac measures $\delta_x$, $\delta_0$. The mass $x \mapsto \bbM(\delta_x - \delta_0)$ is discontinuous and has zero gradient and is therefore unsuitable as a distance between currents. While the Wasserstein-1 metric $x \mapsto \cW(\delta_x, \delta_0)$ is continuous in $x$, it does not easily generalize from probability measures to $k$-currents. In contrast, the flat metric $x \mapsto \bbF_\lambda(\delta_x, \delta_0)$ has a meaningful geometric interpretation also for arbitrary $k$-currents. In Fig.~\ref{fig:distances_2d} 
we illustrate the flat norm for two $1$-currents. In that figure, if $S$ and $T$ are of length one and are $\varepsilon$ apart, then $\bbF_\lambda(S, T) \leq (1 + 2 \lambda) \varepsilon$ which converges to zero for $\varepsilon \to 0$.

Note that for $0$-currents, the flat norm \eqref{eq:flatnorm_sup} is strongly related to the Wasserstein-1 distance
except for the additional constraint on the dual variable $\norm{\omega(x)}^* \leq \lambda$, which in the example of Fig.~\ref{fig:distances} controls the truncation cutoff.
Notice also the similarity of \eqref{eq:flatnorm_min1} to the Beckmann formulation of the Wasserstein-1 distance
\cite{beckmann1952continuous,San15}, with the difference being the implementation of the ``divergence constraint'' with a soft penalty $\lambda \bbM(T - \partial B)$.
Considering the case $\lambda = \infty$ as in the Wasserstein distance is problematic in case we have $k > 0$, since not every current $T \in \cD_k(\R^n)$ is the boundary of a $(k+1)$-current,
see the example above in Fig.~\ref{fig:distances_2d}.

The following proposition studies the effect of the scale parameter $\lambda > 0$ on
the flat norm.
\begin{prop}
  For any $\lambda > 0$, the following relation holds
  \begin{equation}
    \min \{ 1, \lambda \} \cdot \bbF(T) \leq \bbF_\lambda(T) \leq \max \{ 1, \lambda \} \cdot \bbF(T),
    \label{eq:equiv_norms}
  \end{equation}
  meaning that $\bbF$ and $\bbF_\lambda$ are equivalent norms. 
  \label{prop:equiv}
\end{prop}
\begin{proof}
  By a result of \citet{morgan2007} we have the interesting relation
  \begin{equation}
    \bbF_\lambda(T) = \lambda^k \, \bbF(d_{\lambda^{-1}\sharp}T),
  \end{equation}
  where $d_{\lambda}$ is the $\lambda$-dilation.
  Using the bound $\bbF(f_\sharp T) \leq \sup \{ \Lip(f)^k, \Lip(f)^{k+1} \} \bbF(T)$, §4.1.14 in \citet{federer},
  and the fact that $\Lip(d_{\lambda^{-1}}) = \lambda^{-1}$, one inequality directly follows.
  For the other side, notice that
  \begin{equation}
    \begin{aligned}
      \bbF(T) &= \bbF(d_{\lambda\sharp} d_{\lambda^{-1}\sharp} T) = \bbF_{\lambda^{-1}}(d_{\lambda^{-1}\sharp} T) \lambda^k \\
      &\leq \max \{1, \lambda^{-1}\} \bbF(d_{\lambda^{-1}\sharp} T) \lambda^k \\
      &= \max\{ 1, \lambda^{-1} \} \bbF_\lambda(T).
    \end{aligned}
  \end{equation}
  and dividing by $\max\{ 1, \lambda^{-1} \}$ yields the result.
\end{proof}

\begin{figure}[t!]
  \centering
  \begin{center}
  \begin{tabular}{ccc}
    {\def\svgwidth{0.29\linewidth} 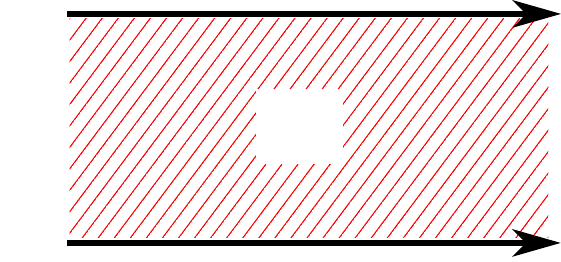}
    &{\def\svgwidth{0.27\linewidth} 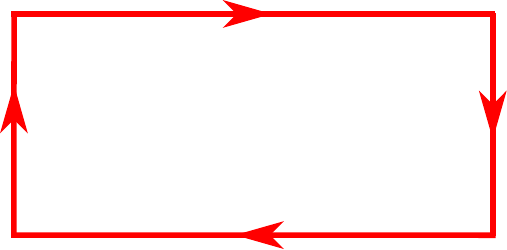}
    &{\def\svgwidth{0.27\linewidth} 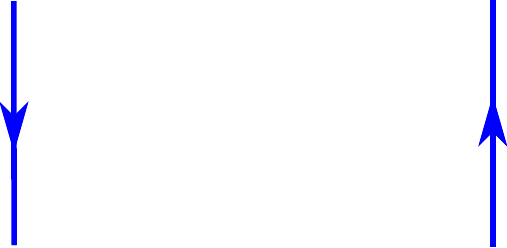}
  \end{tabular}
  \end{center}
  \caption{The flat metric $\bbF_\lambda(S, T)$ is given an optimal decomposition $S - T = {\color{blue} A} + {\color{red} \partial B}$ into a $k$-current {\color{blue} $A$} and the boundary of a $(k+1)$-current {\color{red} $B$} with minimal weighted mass ${\color{blue} \lambda \bbM(A)} + {\color{red} \bbM(B)}$. An intuition is that $\color{blue} \lambda \bbM(A)$ is a penalty that controls how closely $\color{red} \partial B$ should approximate $S - T$, while $\color{red} \bbM(B)$ is the $(k+1)$-dimensional volume of $\color{red} B$.}
  \label{fig:distances_2d}
\end{figure}

The importance of the flat norm is due to the fact that it metrizes the weak$^*$-convergence \eqref{eq:weaks} on compactly supported normal currents with uniformly bounded mass and boundary mass.
\begin{prop}
  Let $\cX \subset \R^d$ be a compact set and $c > 0$ some fixed constant. For a sequence $\{ T_j \} \subset \bN_{k,\cX}(\R^d)$ with $\bbM(T_j) + \bbM(\partial T_j) < c$ we have that:
  \begin{equation}
    \bbF_\lambda(T, T_j) \to 0 \quad \text{ if and only if } \quad T_j \weaks T.
  \end{equation}
\end{prop}
\begin{proof}
  Due to Prop.~\ref{prop:equiv} it is enough to consider the case $\lambda = 1$, which
  is given by Corollary~7.3 in the paper of \citet{federer1960normal}.
\end{proof}

%% file: figures/massnorm.tex
%
%
\begin{tikzpicture}

\begin{axis}[%
width=0.9in,
height=0.6in,
scale only axis,
xmin=-3,
xmax=3,
ymin=0,
ymax=3,
axis background/.style={fill=white},
ticks=none,
xlabel={$x$},
ylabel={$\bbM$},
axis lines = middle,
]

\addplot [color=black]
  table[row sep=crcr]{%
-3	2\\
-0.170000000000001	2\\
};

\addplot [color=black, draw=none, mark=o, mark options={solid, black}]
  table[row sep=crcr]{%
0	2\\
};

\addplot [color=black, draw=none, mark=*, mark options={solid, black}]
  table[row sep=crcr]{%
0	0\\
};

\addplot [color=black]
  table[row sep=crcr]{%
0.170000000000001	2\\
3	2\\
};

\end{axis}
\end{tikzpicture}%

%% file: figures/wasserstein.tex
%
%
\begin{tikzpicture}

\begin{axis}[%
width=0.9in,
height=0.6in,
scale only axis,
xmin=-3,
xmax=3,
ymin=0,
ymax=3,
axis background/.style={fill=white},
ticks = none,
xlabel={$x$},
ylabel={$\cW_1$},
x label style={at={(0.5,0.2)}},
y label style={at={(0.41,0.5)}},
axis lines = middle,
]

\addplot [color=black]
  table[row sep=crcr]{%
-3	3\\
0	0\\
3 3\\
};

\addplot [color=transparent, draw=none, mark=*, mark options={white}, opacity=0]
  table[row sep=crcr]{%
0	0\\
};

\end{axis}
\end{tikzpicture}%

%% file: figures/flatnorm.tex
%
%
\begin{tikzpicture}

\begin{axis}[%
width=0.9in,
height=0.6in,
scale only axis,
xmin=-3,
xmax=3,
ymin=0,
ymax=3,
axis background/.style={fill=white},
ticks = none,
xlabel={$x$},
ylabel={$\bbF_\lambda$},
x label style={at={(0.5,0.2)}},
y label style={at={(0.41,0.5)}},
axis lines = middle,
]

\addplot [color=black]
  table[row sep=crcr]{%
-3	1\\
-1	1\\
};

\addplot [color=black]
  table[row sep=crcr]{%
-1	1\\
0	0\\
1	1\\
};

\addplot [color=black]
  table[row sep=crcr]{%
1	1\\
3	1\\
};

\addplot [color=transparent, draw=none, mark=*, mark options={white}, opacity=0]
  table[row sep=crcr]{%
0	0\\
};

\end{axis}
\end{tikzpicture}%

%% file: figures/flatnorm/f1.pdf_tex
\begingroup%
  \makeatletter%
  \providecommand\color[2][]{%
    \errmessage{(Inkscape) Color is used for the text in Inkscape, but the package 'color.sty' is not loaded}%
    \renewcommand\color[2][]{}%
  }%
  \providecommand\transparent[1]{%
    \errmessage{(Inkscape) Transparency is used (non-zero) for the text in Inkscape, but the package 'transparent.sty' is not loaded}%
    \renewcommand\transparent[1]{}%
  }%
  \providecommand\rotatebox[2]{#2}%
  \ifx\svgwidth\undefined%
    \setlength{\unitlength}{161.55661621bp}%
    \ifx\svgscale\undefined%
      \relax%
    \else%
      \setlength{\unitlength}{\unitlength * \real{\svgscale}}%
    \fi%
  \else%
    \setlength{\unitlength}{\svgwidth}%
  \fi%
  \global\let\svgwidth\undefined%
  \global\let\svgscale\undefined%
  \makeatother%
  \begin{picture}(1,0.46638939)%
    \put(0,0){\includegraphics[width=\unitlength,page=1]{f1.pdf}}%
    \put(0.46829069,0.18776599){\color[rgb]{0,0,0}\makebox(0,0)[lb]{\smash{$\color{red} B$}}}%
    \put(0.00024922,0.39123666){\color[rgb]{0,0,0}\makebox(0,0)[lb]{\smash{$S$}}}%
    \put(-0.0041104,0.00727784){\color[rgb]{0,0,0}\makebox(0,0)[lb]{\smash{$T$}}}%
  \end{picture}%
\endgroup%

%% file: figures/flatnorm/f2.pdf_tex
\begingroup%
  \makeatletter%
  \providecommand\color[2][]{%
    \errmessage{(Inkscape) Color is used for the text in Inkscape, but the package 'color.sty' is not loaded}%
    \renewcommand\color[2][]{}%
  }%
  \providecommand\transparent[1]{%
    \errmessage{(Inkscape) Transparency is used (non-zero) for the text in Inkscape, but the package 'transparent.sty' is not loaded}%
    \renewcommand\transparent[1]{}%
  }%
  \providecommand\rotatebox[2]{#2}%
  \ifx\svgwidth\undefined%
    \setlength{\unitlength}{145.93155518bp}%
    \ifx\svgscale\undefined%
      \relax%
    \else%
      \setlength{\unitlength}{\unitlength * \real{\svgscale}}%
    \fi%
  \else%
    \setlength{\unitlength}{\svgwidth}%
  \fi%
  \global\let\svgwidth\undefined%
  \global\let\svgscale\undefined%
  \makeatother%
  \begin{picture}(1,0.4914527)%
    \put(0,0){\includegraphics[width=\unitlength,page=1]{f2.pdf}}%
    \put(0.37774982,0.2011108){\color[rgb]{0,0,0}\makebox(0,0)[lb]{\smash{$\color{red} \partial B$}}}%
  \end{picture}%
\endgroup%

%% file: figures/flatnorm/f3.pdf_tex
\begingroup%
  \makeatletter%
  \providecommand\color[2][]{%
    \errmessage{(Inkscape) Color is used for the text in Inkscape, but the package 'color.sty' is not loaded}%
    \renewcommand\color[2][]{}%
  }%
  \providecommand\transparent[1]{%
    \errmessage{(Inkscape) Transparency is used (non-zero) for the text in Inkscape, but the package 'transparent.sty' is not loaded}%
    \renewcommand\transparent[1]{}%
  }%
  \providecommand\rotatebox[2]{#2}%
  \ifx\svgwidth\undefined%
    \setlength{\unitlength}{145.93110352bp}%
    \ifx\svgscale\undefined%
      \relax%
    \else%
      \setlength{\unitlength}{\unitlength * \real{\svgscale}}%
    \fi%
  \else%
    \setlength{\unitlength}{\svgwidth}%
  \fi%
  \global\let\svgwidth\undefined%
  \global\let\svgscale\undefined%
  \makeatother%
  \begin{picture}(1,0.48763505)%
    \put(0,0){\includegraphics[width=\unitlength,page=1]{f3.pdf}}%
    \put(0.13856918,0.28021992){\color[rgb]{0,0,0}\makebox(0,0)[lb]{\smash{${\color{blue}A} = S - T$}}}%
    \put(0.43862031,0.09884287){\color[rgb]{0,0,0}\makebox(0,0)[lb]{\smash{$- \color{red}\partial B$}}}%
  \end{picture}%
\endgroup%

%% file: sections/gan.tex
\section{Flat Metric Minimization}
\label{sec:gan}
Motivated by the theoretical properties of the flat metric shown in the previous section,
we consider the following optimization problem:
\begin{equation}
  \min_{\theta \in \Theta} ~ \bbF_\lambda(g_{\theta\sharp} S, T),
  \label{eq:primal_problem}
\end{equation}
where $S \in \bN_{k,\cZ}(\R^l)$ and $T \in \bN_{k, \cX}(\R^d)$.
We will assume that $g : \cZ \times \Theta \to \cX$ is parametrized
with parameters in a compact set $\Theta \subset \R^n$ and write $g_\theta : \cZ \to \cX$
to abbreviate $g(\cdot, \theta)$ for some $\theta \in \Theta$. We need the following
assumption to be able to prove the existence of minimizers for the problem \eqref{eq:primal_problem}.
\begin{asm}
  The map $g : \cZ \times \Theta \to \cX$ is smooth in $z$ with uniformly bounded derivative.
  Furthermore, we assume that $g(z, \cdot)$ is locally Lipschitz continuous and that the parameter set $\Theta \subset \R^n$ is compact.
  \label{asm:1}
\end{asm}
Under this assumption, we will show that the objective in \eqref{eq:primal_problem} is Lipschitz continuous. 
This will in turn guarantee existence of minimizers, as the domain is assumed to be compact.
\begin{prop}
  Let $S \in \bN_{k, \cZ}(\R^l)$, $T \in \bN_{k, \cX}(\R^d)$ be normal currents with compact support. If the pushforward map $g : \cZ \times \Theta \to \cX$ fulfills Assumption~\ref{asm:1}, then the function $\theta \mapsto \bbF_\lambda(g_{\theta\sharp} S, T)$ is Lipschitz continuous and hence differentiable almost everywhere.
  \label{prop:cnt}
\end{prop}
\begin{proof}
  In Appendix~\ref{app:a}.
\end{proof}

\subsection{Application to Generative Modeling}
We now turn towards our considered application illustrated in Fig.~\ref{fig:nutshell}. There, we denote by $k \geq 0$ the number of tangent vectors we specify at each sample point.
The latent current $S \in \bN_{k,\cZ}(\R^l)$ is constructed by combining a probability distribution $\mu \in \bN_{0, \cZ}(\R^l)$,
which could for example be the uniform distribution,
with the unit $k$-vectorfield as follows:
\begin{equation}
  S = \mu \wedge (e_1 \wedge \hdots \wedge e_k).
  \label{eq:S}
\end{equation}
For an illustration, see the right side of Fig.~\ref{fig:nutshell} and Fig.~\ref{fig:currents}.
The data current $T \in \bN_{k,\cX}(\R^d)$ is constructed from the samples $\{ x_i \}_{i=1}^N$
and tangent vectorfields $T_i : \cX \to \LM_k \R^d$. 
\begin{equation}
  T = \frac{1}{N} \sum_{i=1}^N \delta_{x_i} \wedge T_i,
  \label{eq:T}
\end{equation}
The tangent $k$-vectorfields $T_i(x) = T_{i,1} \wedge \hdots \wedge T_{i,k}$
are given by individual tangent vectors to the data manifold $T_{i,j} \in \R^d$.
For an illustration, see the left side of Fig.~\ref{fig:nutshell} or Fig.~\ref{fig:currents}.
After solving \eqref{eq:primal_problem}, the
map $g_\theta : \cZ \to \cX$ will be our generative model, where changes in the latent
space $\cZ$ along the unit directions $e_1, \hdots, e_k$ are expected to behave equivariantly
to the specified tangent directions $T_{i,1}, \hdots, T_{i,k}$ near $g(z)$.

\setlength{\tabcolsep}{12pt}
\begin{figure*}[t!]
  \centering
  \begin{center}
  \begin{tabular}{rccccc}
    \rotatebox{90}{\parbox[t]{1in}{\hspace*{0.7cm}$k=0$}}
    &{\includegraphics[width=0.145\linewidth]{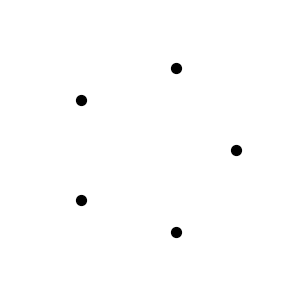}}
    &{\includegraphics[width=0.145\linewidth]{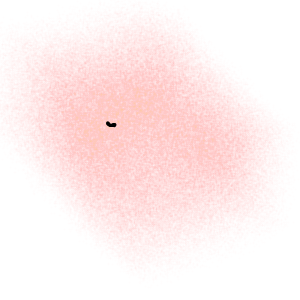}}
    &{\includegraphics[width=0.145\linewidth]{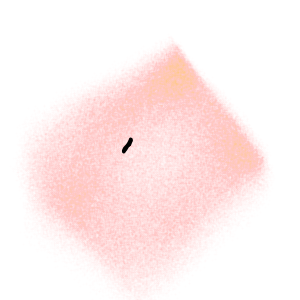}}
    &{\includegraphics[width=0.145\linewidth]{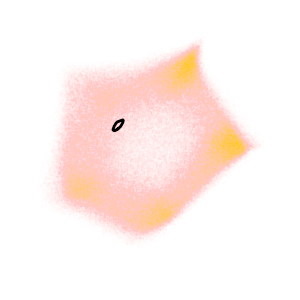}}
    &{\includegraphics[width=0.145\linewidth]{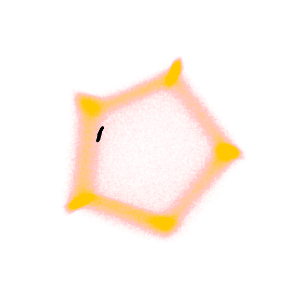}}\\[-0.45cm]
    \rotatebox{90}{\parbox[t]{1in}{\hspace*{0.7cm}$k=1$}}
    &{\includegraphics[width=0.145\linewidth]{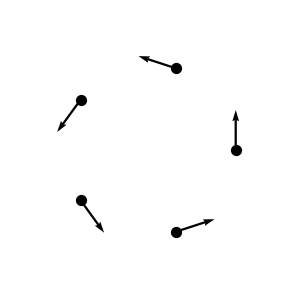}}
    &{\includegraphics[width=0.145\linewidth]{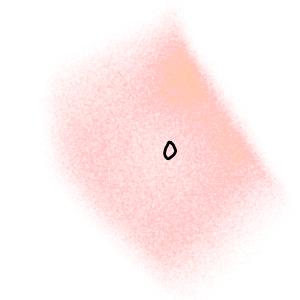}}
    &{\includegraphics[width=0.145\linewidth]{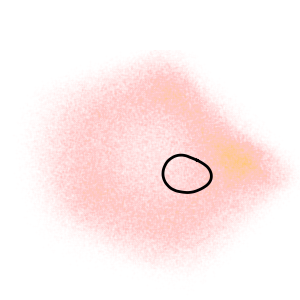}}
    &{\includegraphics[width=0.145\linewidth]{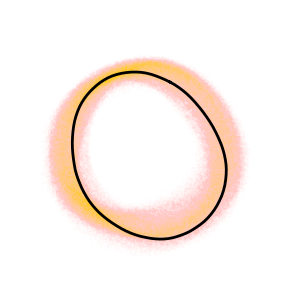}}
    &{\includegraphics[width=0.145\linewidth]{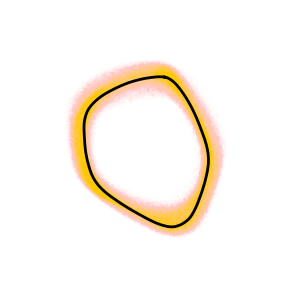}}\\
    &$T \in \bN_{k, \cX}(\R^2)$
    &Epoch $250$
    &Epoch $500$
    &Epoch $1000$
    &Epoch $2000$
  \end{tabular}
  \end{center}
  \caption{We illustrate the effect of moving from $k = 0$ to $k = 1$ and plot the measure $\norm{g_\sharp S}$ of the pushforward of a $k$-current $S \in \bN_{k, \cZ}(\R^5)$ (shown in orange) for different epochs. The black curve illustrates a walk along the first latent dimension $z_1$. For $k = 0$, which is similar to WGAN-GP \cite{GAADC17},
  the latent walk is not meaningful. The proposed approach ($k = 1$) allows to specify tangent vectors at the samples to which the first latent dimension behaves equivariantly, yielding an interpretable representation.}
  \label{fig:exp2d}
\end{figure*}

%% file: sections/impl.tex
\subsection{FlatGAN Formulation}
\label{sec:impl}
To get a primal-dual formulation (or two player zero-sum game) in the spirit of GANs, we insert the definition of the flat norm \eqref{eq:flatnorm_sup} into the primal problem \eqref{eq:primal_problem}:
\begin{equation}
  \min_{\theta \in \Theta} ~ \sup_{\substack{\omega \in \cD^k(\R^d)\\ \norm{\omega}^* \leq \lambda, \norm{d\omega}^* \leq 1}} S({g_\theta}^\sharp \omega) - T(\omega),
  \label{eq:primal_dual_problem}
\end{equation}
where $\theta \in \Theta$ are for example the parameters of a neural network.
In the above equation, we also used the definition of pushforward \eqref{eq:pushfwd}.
Notice that for $k = 0$ the exterior derivative in \eqref{eq:primal_dual_problem} specializes to the gradient. This yields
a Lipschitz constraint, and as for sufficiently large $\lambda$ the other constraint becomes irrelevant, 
the problem \eqref{eq:primal_dual_problem} is
closely related to the Wasserstein GAN \cite{BALO17}. The
novelty in this work is the generalization to $k > 0$.

Combining \eqref{eq:S} and \eqref{eq:T} into \eqref{eq:primal_dual_problem} we arrive at the objective
\begin{equation}
  \begin{aligned}
    &E(\theta, \omega) = - \frac{1}{N} \sum_{i=1}^N \iprod{\omega(x_i)}{T_i} \\
    &+ \int \iprod{\omega \circ g_\theta}{(\nabla_z g_\theta \cdot e_1) \wedge \hdots \wedge (\nabla_z g_\theta \cdot e_k)} \, \dd \mu.
  \end{aligned}
  \label{eq:pd_final}
\end{equation}
Interestingly, due to the pullback, the discriminator $\omega$ inspects not only the output of the generator,
but also parts of its Jacobian matrix. As a remark, relations between the generator Jacobian and GAN performance
have recently been studied by \citet{odena2018generator}.

The constraints in \eqref{eq:primal_dual_problem} are implemented using penalty terms. First notice that due to the
definition of the comass norm \eqref{eq:comass}, the first constraint is equivalent to
imposing $|\iprod{\omega(x)}{v}| \leq \lambda$ for all simple $k$-covectors with $|v| = 1$.
We implement this with the a penalty term with parameter $\rho > 0$ as follows:
\begin{equation}
  \rho \cdot \int_\cX \int \max \{ 0, |\iprod{\omega(x)}{v}| - \lambda \}^2 \, \dd \gamma_{k,d}(v) \, \dd x,
  \label{eq:pen1}
\end{equation}
where $\gamma_{k,d}$ denotes the Haar measure on the Grassmannian manifold $\mathbf{Gr}(d, k) \subset \LM_k \R^d$ of $k$-dimensional subspaces in $\R^d$, see Chapter 3.2~in~\citet{KP08}.
Similarly, the constraint on the exterior derivative is implemented by
another penalty term as follows:
\begin{equation}
  \rho \cdot \int_\cX \int \max \{ 0, |\iprod{d \omega(x)}{v}| - 1 \}^2 \, \dd \gamma_{k+1,d}(v) \, \dd x.
  \label{eq:pen2}
\end{equation}

\subsection{Implementation with Deep Neural Networks}
For high dimensional practical problems it is completely infeasible to directly work with $\LM_k \R^d$ due to the curse of dimensionality. For example, already for the MNIST dataset augmented with two tangent vectors ($k = 2$, $d = 28^2$), we have
that $\dim(\LM_k \R^d) \approx 3 \cdot 10^5$. 

\setlength{\tabcolsep}{1pt}
\begin{figure}[t!]
  \centering
  \begin{center}
  \begin{tabular}{cc}
    {\includegraphics[width=0.49\linewidth]{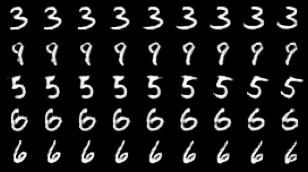}}
    & {\includegraphics[width=0.49\linewidth]{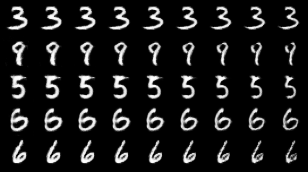}} \\
    varying $z_1$ (rotation)
    & varying $z_2$ (thickness)
  \end{tabular}
  \end{center}
  \caption{We show the effect of varying the first two components in $128$-dimensional
    latent space, corresponding to the two selected tangent vectors which are rotation and thickness. As seen in the figure, varying the corresponding latent representation yields an interpretable effect on the output, corresponding to the specified tangent direction.}
  \label{fig:exp_mnist}
\end{figure}

To overcome this issue,  we unfortunately have to resort to a few heuristic approximations. To that end, we first notice that in the formulations the dual variable $\omega : \R^d \to \LM^k \R^d$ only appears as an inner product with simple $k$-vectors, so we can implement it by implicitly describing its action, i.e., interpret it as a map $\omega : \R^d \times \LM_k \R^d \to \R$:
\begin{align}
  \label{eq:omega_struct}
  &\omega(x, v_1 \wedge \hdots \wedge v_k) \\
  &\quad =\omega^0(x) + \alpha \iprod{\omega^{1,1}(x) \wedge \hdots \wedge \omega^{1,k}(x)}{v_1 \wedge \hdots \wedge v_k}, \notag 
\end{align}
Theoretically, the ``affine term'' $\omega^0(x)$ is not fully justified as the map does not describe an inner product on $\LM_k \R^d$ anymore,
but we found it to improve the quality of the generative model.
An attempt to justify this in the context of GANs is that the function $\omega^0 : \R^d \to \R$ is the usual ``discriminator''
while the $\omega^{1,i} : \R^d \to \R^d$ are combined to discriminate oriented tangent planes.

In practice, we parametrize $\omega^0$, $\omega^{1,i}$ using deep neural networks. For efficiency reasons, 
the networks share their parameters up until the last few layers.

The inner product in \eqref{eq:omega_struct} between the simple vectors is implemented by a $k \times k$-determinant, see \eqref{eq:det}. 
The reason we do this is to satisfy the properties of the Grassmann algebra \eqref{eq:gm1} -- \eqref{eq:gm2}.
This is important, since otherwise the ``discriminator'' $\omega$ could distinguish between different representations of the same oriented tangent plane.

\begin{figure*}[t!]
  \centering
  \begin{center}
    \begin{tabular}{ccc}
    {\includegraphics[width=0.33\linewidth]{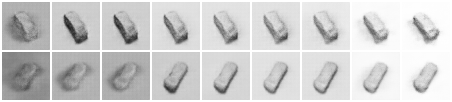}}
    &{\includegraphics[width=0.33\linewidth]{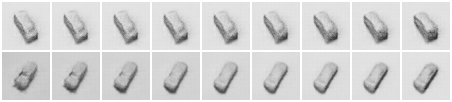}}
    &{\includegraphics[width=0.33\linewidth]{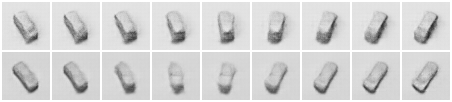}}\\
    varying $z_1$ (lighting)
    &varying $z_2$ (elevation)
    &varying $z_3$ (azimuth)
    \end{tabular}
  \end{center}
  \caption{From left to right we vary the latent codes in $[-1, 1]$ after training on the smallNORB dataset \cite{lecun2004}.} 
  \label{fig:exp_smallnorb}
\end{figure*}

For the implementation of the penalty term \eqref{eq:pen2}, we use the definition of the exterior derivative \eqref{eq:ext_diff_lim} together 
with the ``approximate form'' \eqref{eq:omega_struct}. 
To be compatible with the affine term we use a seperate penalty on $\omega^0$, which we also found to give better results:
\begin{equation}
  \begin{aligned}
    &|d \omega(x, v_1 \wedge \hdots \wedge v_{k+1})| \approx (k + 1) \norm{\nabla_x \omega^0(x)} \\
    &\qquad + \alpha \left| \sum_{i=1}^{k+1} (-1)^{i-1} \nabla_x \det(W(x)^\top V_i) \cdot v_i \right|.
  \end{aligned}
  \label{eq:ext_diff_bound}
\end{equation}
In the above equation, $V_i \in \R^{d \times k}$ is the matrix with columns given by the vectors $v_1, \hdots, v_{k+1}$ but with $v_i$ omitted and
$W(x) \in \R^{d \times k}$ is the matrix with columns given by the $\omega^{1,i}(x)$. 
Another motivation for this implementation is, that in the case $k = 0$ the second term in \eqref{eq:ext_diff_bound} disappears and one recovers the well-known ``gradient penalty'' regularizer proposed by \citet{GAADC17}.

For the stochastic approximation of the penalty terms \eqref{eq:pen1} -- \eqref{eq:pen2} we sample from the Haar measure on the Grassmannian (i.e., taking random $k$-dimensional and $(k+1)$-dimensional subspaces in $\R^d$) by computing singular value decomposition of random $k \times d$ Gaussian matrices. Furthermore, we found it beneficial in practice to enforce the penalty terms only at the data points as for example advocated in the recent work \cite{mescheder}.
The right multiplied Jacobian vector products (also referred to as ``{\tt rop}'' in some frameworks) in \eqref{eq:ext_diff_bound} as well as in the loss function \eqref{eq:pd_final} are implemented using two additional backpropagations.

%% file: sections/experiments.tex
\begin{figure}[t]
  \centering
  \begin{center}
      {\includegraphics[width=0.99\linewidth]{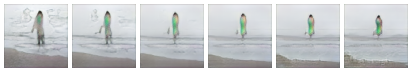}}\\[-0.14cm]
      {\includegraphics[width=0.99\linewidth]{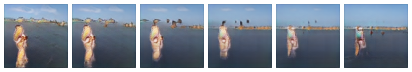}}\\
      varying $z_1$ (time)
  \end{center}
  \caption{Varying the learned latent representation of time. The model captures behaviours such as people walking on the beach, see also the results shown in Fig.~\ref{fig:teaser}.}
  \label{fig:videos}
\end{figure}

\section{Experiments}
\label{sec:exp}
The specific hyperparameters, architectures and tangent vector setups used in practice\footnote{See \url{https://github.com/moellenh/flatgan} for a {\tt PyTorch} implementation to reproduce Fig.~\ref{fig:exp2d} and Fig.~\ref{fig:exp_mnist}.}
 are detailed in Appendix~\ref{app:b}. 

\subsection{Illustrative 2D Example}
As a first proof of concept, we illustrate the effect of moving from $k = 0$ to $k = 1$ on
a very simple dataset consisting of five points on a circle. As shown in
Fig.~\ref{fig:exp2d}, for $k = 0$ (corresponding to a WGAN-GP formulation) varying the first latent
variable has no clear meaning. In contrast, with the proposed FlatGAN formulation, we
can specify vectors tangent to the circle from which the data is sampled.
This yields an interpretable latent representation that corresponds to an angular movement along the circle.
As the number of epochs is increasing, both formulations tend to concentrate most of the probability mass on the five 
data points. However, since $g_\theta : \cZ \to \cX$ is continuous by construction an interpretable path remains.

\subsection{Equivariant Representation Learning}
In Fig.~\ref{fig:exp_mnist} and Fig.~\ref{fig:exp_smallnorb} we show examples for $k = 2$ and $k = 3$
on MNIST respectively the smallNORB dataset of \citet{lecun2004}. For MNIST, we compute the tangent vectors
manually by rotation and dilation of the digits, similar as done by \citet{simard1992tangent,simard1998transformation}.
For the smallNORB example, the tangent vectors are given as differences between the corresponding images.
As observed in the figures, the proposed formulation leads to interpretable latent codes which behave
equivariantly with the generated images.
We remark that the goal was not to achieve state-of-the-art image
quality but rather to demonstrate that specifying tangent vectors 
yields disentangled representations.  As remarked by
\citet{spatial}, representing a 3D scene with a sequence of 2D
convolutions is challenging and a specialized architecture based on a
voxel representation would be more appropriate for the smallNORB example.

\subsection{Discovering the Arrow of Time}
In our last experiment, we set $k = 1$ and specify the tangent vector as
the difference of two neighbouring frames in video data. We train on the tinyvideo
beach dataset \cite{videoGAN}, which consists of more than 36 million frames.
After training for about half an epoch, we can already observe a learned latent representation of time,
see Fig.~\ref{fig:teaser} and Fig.~\ref{fig:videos}. 
We generate individual frames by varying the latent coordinate $z_1$ 
from $-12.5$ to $12.5$. 

Even though the model is trained on individual frames in 
random order, a somewhat coherent representation of time is discovered which captures 
phenomena such as ocean waves or people walking on the beach.

%% file: sections/discussion.tex
\section{Discussion and Conclusion}
\label{sec:discussion}
In this work, we demonstrated that $k$-currents can be used introduce a
notion of orientation into probabilistic models.
Furthermore, in experiments we have shown that specifying partial tangent information to
the data manifold leads to interpretable and equivariant latent representations such as 
the camera position and lighting in a 3D scene or the arrow of time in time series data.

The difference to purely unsupervised approaches such as InfoGAN or $\beta$-VAE 
is, that we can encourage potentially very complex latent representations to be learned. Nevertheless, an additional 
mutual information term as in \cite{CDHSSA16} can be directly added to the formulation so that some representations could be encouraged
through tangent vectors and the remaining ones are hoped to be discovered in an unsupervised fashion.

Generally speaking,
we believe that geometric measure theory is a rather underexploited field
with many possible application areas in probabilistic machine learning. We see this work
as a step towards leveraging this potential.

\section*{Acknowledgements} We thank Kevin R. Vixie for his detailed feedback and comments on the manuscript.
The work was partially supported by
the German Research Foundation (DFG); project 394737018 ``\emph{Functional Lifting 2.0 -- Efficient Convexifications for Imaging and Vision}''.

%% file: sections/appendix.tex
{ \Large \textbf{Appendix} }
\setcounter{section}{0}
\section{Proof of Proposition~\ref{prop:cnt}}
\label{app:a}
Since $g_{\theta\sharp}S$ and $T$ are normal currents we know
$\bbF_\lambda(g_{\theta\sharp} S, T) < \infty$ for all $\theta \in \Theta$.

We now directly show Lipschitz continuity. First notice that
\begin{align}
  &\bbF_\lambda(g_{\theta\sharp} S - T) = \bbF_\lambda(g_{\theta\sharp} S + g_{\theta'\sharp} S - g_{\theta'\sharp} S- T) \\
  &\qquad \leq \bbF_\lambda(g_{\theta\sharp} S - g_{\theta' \sharp} S) + \bbF_\lambda(g_{\theta'\sharp} S - T),
\end{align}
yields the following bound:
\begin{equation}
  | \bbF_\lambda(g_{\theta\sharp} S - T) - \bbF_\lambda(g_{\theta' \sharp} S - T) | \leq \bbF_\lambda(g_{\theta\sharp} S - g_{\theta' \sharp}S).
  \label{eq:tri_ineq}
\end{equation}
Due to Prop.~\ref{prop:equiv} we have that 
\begin{equation}
  \bbF_\lambda(g_{\theta\sharp} S - g_{\theta' \sharp}S) \leq \max \{ 1, \lambda \} \cdot \bbF(g_{\theta\sharp} S - g_{\theta' \sharp}S).
  \label{eq:equiv_norm}
\end{equation}

Now define the compact set $C \subset \R^d$ as 
\begin{equation}
  \begin{aligned}
    C = \bigl \{ (1 - t) g_\theta(z) + t g_{\theta'}(z) : ~ &z \in \spt S, \\
    &0 \leq t \leq 1  \bigr \},
  \end{aligned}
\end{equation}
and as in §4.1.12 in \citet{federer} for compact $K \subset \R^d$
the ``stronger'' flat norm
\begin{equation}
  \begin{aligned}
    &\bbF_K(T)
    = \sup \bigl \{ T(\omega) ~|~ \omega \in \cD^k(\R^d), \text{ with } \\
    & \hspace{0.25cm} \norm{\omega(x)}^* \leq 1, \norm{d\omega(x)}^* \leq 1 \text{ for all } x \in K \bigr \}.
  \end{aligned}
  \label{eq:flatnorm_C}
\end{equation}

Since the constraint in the supremum in \eqref{eq:flatnorm_C} is less restrictive than in the definition of the flat norm \eqref{eq:flatnorm_sup}, we have
\begin{equation}
  \bbF(g_{\theta\sharp} S - g_{\theta' \sharp}S) \leq \bbF_C(g_{\theta\sharp} S - g_{\theta' \sharp}S).
  \label{eq:bound0}
\end{equation}

Then, the inequality after §4.1.13 in \citet{federer} bounds the right side of \eqref{eq:bound0} for $k > 0$ by
\begin{equation}
  \begin{aligned}
    &\bbF_C(g_{\theta\sharp} S - g_{\theta' \sharp}S) \leq \\
    &\quad \norm{S} (|g_{\theta} - g_{\theta'}| \rho^k) + \norm{\partial S} (|g_{\theta} - g_{\theta'}| \rho^{k-1}),
  \end{aligned}
  \label{eq:bound1}
\end{equation}
where $\rho(z) = \max \{ \norm{\nabla_z g(z, \theta)}, \norm{\nabla_z g(z, \theta')} \} < \infty$ due to Assumption~\ref{asm:1}
and we write $\norm{S}(f) = \int f(z) \, \mathrm{d} \norm{S}(z)$, where $\norm{S}$ is defined in the sense of \eqref{eq:polar}. For $k = 0$, a similar bound can be derived without the term $\norm{\partial S}$.

For $k > 0$, by setting $\mu_S = \norm{\partial S} + \norm{S}$
we can further bound the term in \eqref{eq:bound1} by
\begin{equation}
  \begin{aligned}
    &\norm{S} (|g_{\theta} - g_{\theta'}| \rho^k) + \norm{\partial S} (|g_{\theta} - g_{\theta'}| \rho^{k-1}) \leq \\
    &\qquad c_1 \cdot \int \norm{g_{\theta}(z) - g_{\theta'}(z)} \mathrm{d}\mu_S(z),
  \end{aligned}
  \label{eq:bound2}
\end{equation}
where $c_1 = \sup_z \max \{ \rho^k(z), \rho^{k-1}(z) \}$. For $k = 0$, the bound is derived analogously.

Now since $g(z, \cdot)$ is locally Lipschitz and $\Theta \subset \R^n$ is compact,
$g(z, \cdot)$ is Lipschitz and we denote the constant as $\Lip(g)$, leading to the bound
\begin{equation}
  \int \norm{g_{\theta}(z) - g_{\theta'}(z)} \mathrm{d}\mu_S(z) \leq \mu_S(\cZ) \Lip(g) \cdot \norm{\theta - \theta'}.
  \label{eq:bound3}
\end{equation}
Since $S \in \bN_{k, \cZ}(\R^l)$ is a normal current, $\mu_S(\cZ) < \infty$. Thus by combining \eqref{eq:tri_ineq}, \eqref{eq:equiv_norm}, \eqref{eq:bound0}, \eqref{eq:bound1}, \eqref{eq:bound2} and \eqref{eq:bound3} there is a finite $c_2 = \max\{1, \lambda \} \cdot c_1 \cdot \mu_S(\cZ) \cdot \Lip(g) < \infty$ such that
\begin{equation}
  | \bbF_\lambda(g_{\theta\sharp} S - T) - \bbF_\lambda(g_{\theta' \sharp} S - T) | \leq c_2 \norm{\theta - \theta'}.
\end{equation}
Therefore, the cost $\bbF_\lambda(g_{\theta\sharp} S, T)$ in \eqref{eq:primal_problem} is Lipschitz in $\theta$ and by Rademacher's theorem, §3.1.6 in \citet{federer}, also differentiable almost everywhere.

\section{Parameters and Network Architectures}
\label{app:b}
For all experiments we use Adam optimizer \cite{kingma2014adam}, with step size $10^{-4}$
and momentum parameters $\beta_1 = 0.5$, $\beta_2 = 0.9$. The batch size
is set to $50$ in all experiments except the first one (which runs full batch with batch size $5$). We always set $\lambda = 1$.

\subsection{Illustrative 2D Example}
We pick the same parameters for $k \in \{ 0, 1 \} $. We set the penalty to $\rho = 10$ and use $5$ discriminator updates per generator update as in \cite{GAADC17}.
The generator is a $5$ -- $6$ -- $250$ -- $250$ -- $250$ -- $2$ fully connected network
with leaky ReLU activations. The first layer ensures that the latent coordinate $z_1$
has the topology of a circle, i.e., it is implemented as $(\cos(z_1), \sin(z_1), z_2, z_3, z_4, z_5)$.
The discriminators $\omega^0$ and $\omega^{1,1}$ are $2$ -- $100$ -- $100$ -- $100$ -- $1$
respectively $2$ -- $100$ -- $100$ -- $2$ nets with leaky ReLUs.
The distribution on the latent is a uniform $z_1 \sim U([-\pi, \pi])$ and $z_{i} \sim \mathcal{N}(0, 1)$ for the remaining $4$ latent codes.

\subsection{MNIST}
For the remaining experiments, we use only $1$ discriminator update per iteration.
The digits are resized to $32 \times 32$. For generator we use DCGAN architecture \cite{dcgan} without batch norm and with ELU activations, see Table~\ref{tab:mnist_gen}.
\setlength{\tabcolsep}{8pt}
\begin{table}[h!]
  \centering
  \begin{center}
\begin{tabular}{ccc}
  \toprule
  layer name & output size & filters \\
  \midrule
  Reshape & $128 \times 1 \times 1$ & -- \\
  Conv2DTranspose & $32 F \times 4 \times 4$ & $128 \to 32 F$ \\
  Conv2DTranspose & $16 F \times 8 \times 8$ & $32 F \to 16 F$ \\
  Conv2DTranspose & $4 F \times 16 \times 16$ & $16 F \to 4 F$ \\
  Conv2DTranspose & $1 \times 32 \times 32$ & $4 F \to 1$\\
  \bottomrule
\end{tabular}
\end{center}
\caption{Generator architecture for MNIST experiment, $F = 32$.}
\label{tab:mnist_gen}
\end{table}
The discriminators are given by the architectures in Table~\ref{tab:mnist_disc}, with leaky ReLUs between the layers.
\definecolor{LightRed}{rgb}{1,0.7,0.7}
\definecolor{LightBlue}{rgb}{0.7,0.7,1}
\begin{table}[h!]
  \centering
  \begin{center}
\begin{tabular}{ccc}
  \toprule
  layer name & output size & filters \\
  \midrule
  Reshape & $1 \times 32 \times 32$ & -- \\
  Conv2D & $2F \times 16 \times 16$ & $1 \to 2F$ \\
  Conv2D & $4F \times 8 \times 8$ & $2F \to 4F$ \\
  Conv2D & $32F \times 4 \times 4$ & $4F \to 32F$ \\
  \rowcolor{LightRed}
  Conv2D & $1 \times 1 \times 1$ & $32F \to 1$\\
  \rowcolor{LightBlue}
  Conv2DTranspose & $1 \times 8 \times 8$ & $32F \to 1$\\
  \bottomrule
\end{tabular}
\end{center}
\caption{The discriminator $\omega^0$ has $F = 32$ and red last layer. The discriminators $\omega^{1,1}$, $\omega^{1,2}$ have $F = 8$ and last layer in blue.}
\label{tab:mnist_disc}
\end{table}

Before computing $\iprod{\omega^{1,1}(x) \wedge \omega^{1,2}(x)}{v_1 \wedge v_2}$, the tangent
images $v_1, v_2 \in \R^{32 \cdot 32}$ are convolved with a Gaussian with a standard deviation of $2$ and downsampled to $8 \times 8$
using average pooling. The distributions on the latent space are given by $z_1 \sim U([-7.5, 7.5])$, $z_2 \sim U([-0.5, 0.5])$ and $z_{i} \sim \mathcal{N}(0, 1)$ for the remaining $126$ latent variables. The tangent vectors at each sample are computed by a $2$ degree rotation and a dilation with radius one.

\subsection{SmallNORB}
We downsample the smallNORB images to $48 \times 48$. The architectures
and parameters are chosen similar to the previous MNIST example,
see Table~\ref{tab:norb_gen} and Table~\ref{tab:norb_disc}.

\setlength{\tabcolsep}{8pt}
\begin{table}[h!]
  \centering
  \begin{center}
\begin{tabular}{ccc}
  \toprule
  layer name & output size & filters \\
  \midrule
  Reshape & $128 \times 1 \times 1$ & -- \\
  Conv2DTranspose & $32 F \times 4 \times 4$ & $128 \to 32 F$ \\
  Conv2DTranspose & $16 F \times 8 \times 8$ & $32 F \to 16 F$ \\
  Conv2DTranspose & $16 F \times 12 \times 12$ & $16 F \to 16 F$ \\
  Conv2DTranspose & $4 F \times 24 \times 24$ & $16 F \to 4 F$ \\
  Conv2DTranspose & $1 \times 48 \times 48$ & $4 F \to 1$\\
  \bottomrule
\end{tabular}
\end{center}
\caption{Generator for smallNORB experiment, $F = 24$.}
\label{tab:norb_gen}
\end{table}

\begin{table}[h!]
  \centering
  \begin{center}
\begin{tabular}{ccc}
  \toprule
  layer name & output size & filters \\
  \midrule
  Reshape & $1 \times 48 \times 48$ & -- \\
  Conv2D & $2F \times 24 \times 24$ & $1 \to 2F$ \\
  Conv2D & $4F \times 12 \times 12$ & $2F \to 4F$ \\
  Conv2D & $32F \times 6 \times 6$ & $4F \to 32F$ \\
  \rowcolor{LightRed}
  Conv2D & $1 \times 1 \times 1$ & $32F \to 1$\\
  \rowcolor{LightBlue}
  Conv2DTranspose & $1 \times 12 \times 12$ & $32F \to 1$\\
  \bottomrule
\end{tabular}
\end{center}
\caption{SmallNORB discriminator $\omega^0$, $F = 32$, last layer in shown in red, and tangent discriminators $\omega^{1,1}$, $\omega^{1,2}$, $\omega^{1,3}$ where $F = 8$ and last layer is highlighted in blue.}
\label{tab:norb_disc}
\end{table}

\subsection{Tinyvideos}
The architectures for the tinyvideo experiment are borrowed from the recent work \citet{mescheder},
see Table~\ref{tab:tinyvid_gen} and Table~\ref{tab:tinyvid_disc}.

\setlength{\tabcolsep}{8pt}
\begin{table}[h!]
  \centering
  \begin{center}
\begin{tabular}{ccc}
  \toprule
  layer name & output size & filters \\
  \midrule
  Fully Connected & $8192$ & -- \\
  Reshape & $512 \times 4 \times 4$ & -- \\
  \midrule
  ResNet-Block & $512 \times 4 \times 4$ & $512 \to 512 \to 512$ \\
  NN-Upsampling & $512 \times 8 \times 8$ & -- \\
  \midrule
  ResNet-Block & $256 \times 8 \times 8$ & $512 \to 256 \to 256$ \\
  NN-Upsampling & $256 \times 16 \times 16$ & -- \\
  \midrule
  ResNet-Block & $128 \times 16 \times 16$ & $256 \to 128 \to 128$ \\
  NN-Upsampling & $128 \times 32 \times 32$ & -- \\
  \midrule
  ResNet-Block & $64 \times 32 \times 32$ & $128 \to 64 \to 64$ \\
  NN-Upsampling & $64 \times 64 \times 64$ & -- \\
  \midrule
  ResNet-Block & $64 \times 64 \times 64$ & $64 \to 64 \to 64$ \\
  Conv2D & $3 \times 64 \times 64$ & $64 \to 3$ \\
  \bottomrule
\end{tabular}
\end{center}
\caption{Generator architecture for tinyvideos experiment.}
\label{tab:tinyvid_gen}
\end{table}

\setlength{\tabcolsep}{8pt}
\begin{table}[h!]
  \centering
  \begin{center}
\begin{tabular}{ccc}
  \toprule
  layer name & output size & filters \\
  \midrule
  Conv2D & $64 \times 64 \times 64$ & $3 \to 64$ \\
  \midrule
  ResNet-Block & $64 \times 64 \times 64$ & $64 \to 64 \to 64$ \\
  AvgPool2D & $64 \times 32 \times 32$ & -- \\
  \midrule
  ResNet-Block & $128 \times 32 \times 32$ & $64 \to 64 \to 128$ \\
  AvgPool2D & $128 \times 16 \times 16$ & -- \\
  \midrule
  \rowcolor{LightRed}
  ResNet-Block & $256 \times 16 \times 16$ & $128 \to 128 \to 256$ \\
  \rowcolor{LightRed}
  AvgPool2D & $256 \times 8 \times 8$ & -- \\
  \midrule
  \rowcolor{LightRed}
  ResNet-Block & $512 \times 8 \times 8$ & $256 \to 256 \to 512$ \\
  \rowcolor{LightRed}
  AvgPool2D & $512 \times 4 \times 4$ & -- \\
  \midrule
  \rowcolor{LightRed}
  ResNet-Block & $1024 \times 4 \times 4$ & $512 \to 512 \to 1024$ \\
  \rowcolor{LightRed}
  Conv2D & $1 \times 1 \times 1$ & $1024 \to 1$ \\
  \midrule
  \rowcolor{LightBlue}
  ResNet-Block & $256 \times 16 \times 16$ & $128 \to 256 \to 256$ \\
  \rowcolor{LightBlue}
  Conv2D & $3 \times 16 \times 16$ & $256 \to 3$ \\
  \bottomrule
\end{tabular}
\end{center}
\caption{Discriminator architectures for tinyvideos experiment. Last layers of
  $\omega^0$ are highlighted in red, and the last layers of the temporal discriminator
  $\omega^{1,1}$ are highlighted in blue.}
\label{tab:tinyvid_disc}
\end{table}